%% file: main.tex
\documentclass[sigconf]{aamas} 

\usepackage{balance} 

\input{math_commands}

\usepackage{amsmath,amsthm}
\usepackage{algorithm}
\usepackage{algorithmic}
\usepackage{multirow}
\usepackage{subcaption}
\usepackage{enumitem}
\newtheorem{proposition}{Proposition}
\newcommand{\ours}{VTE} 

\usepackage{cleveref} 
\crefname{section}{Section}{Sections}
\crefname{theorem}{Theorem}{Theorems}
\crefname{corollary}{Corollary}{Corollaries}
\crefname{lemma}{Lemma}{Lemmas}
\crefname{equation}{Eq.}{Equations}
\crefname{proposition}{Proposition}{Propositions}
\crefname{claim}{Claim}{Claims}
\crefname{remark}{Remark}{Remarks}
\crefname{observation}{Observation}{Observations}
\crefname{assumption}{Assumption}{Assumptions}
\crefname{definition}{Definition}{Definitions}
\crefname{appendix}{Appendix}{Appendices}
\crefname{algorithm}{Algorithm}{Algorithms}
\crefname{figure}{Figure}{Figures}
\crefname{table}{Table}{Tables}

\usepackage{booktabs}

\makeatletter
\gdef\@copyrightpermission{
  \begin{minipage}{0.2\columnwidth}
   \href{https://creativecommons.org/licenses/by/4.0/}{\includegraphics[width=0.90\textwidth]{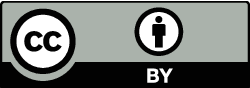}}
  \end{minipage}\hfill
  \begin{minipage}{0.8\columnwidth}
   \href{https://creativecommons.org/licenses/by/4.0/}{This work is licensed under a Creative Commons Attribution International 4.0 License.}
  \end{minipage}
  \vspace{5pt}
}
\makeatother

\setcopyright{ifaamas}
\acmConference[AAMAS '25]{Proc.\@ of the 24th International Conference
on Autonomous Agents and Multiagent Systems (AAMAS 2025)}{May 19 -- 23, 2025}
{Detroit, Michigan, USA}{Y.~Vorobeychik, S.~Das, A.~Nowé  (eds.)}
\copyrightyear{2025}
\acmYear{2025}
\acmDOI{}
\acmPrice{}
\acmISBN{}

\acmSubmissionID{1273}

\title[AAMAS-2025 Formatting Instructions]{On Learning Informative Trajectory Embeddings for\\ Imitation, Classification and Regression}

\author{Zichang Ge$^*$}
\affiliation{
  \institution{Singapore Management University}
  \country{Singapore}
}
\email{zichang.ge.2023@phdcs.smu.edu.sg}

\author{Changyu Chen$^*$}
\affiliation{
  \institution{Singapore Management University}
  \country{Singapore}
}
\email{cychen.2020@phdcs.smu.edu.sg}

\author{Arunesh Sinha}
\affiliation{
  \institution{Rutgers University}
  \country{New Brunswick, NJ, USA}
}
\email{arunesh.sinha@rutgers.edu}

\author{Pradeep Varakantham}
\affiliation{
  \institution{Singapore Management University}
  \country{Singapore}
}
\email{pradeepv@smu.edu.sg}

\begin{abstract}
In real-world sequential decision making tasks like autonomous driving, robotics, and healthcare, learning from observed state-action trajectories is critical for tasks like imitation, classification, and clustering. For example, self-driving cars must replicate human driving behaviors, while robots and healthcare systems benefit from modeling decision sequences, whether or not they come from expert data. Existing trajectory encoding methods often focus on specific tasks or rely on reward signals, limiting their ability to generalize across domains and tasks.

Inspired by the success of embedding models like CLIP and BERT in static domains, we propose a novel method for embedding state-action trajectories into a latent  space that captures the skills and competencies in the dynamic underlying  decision-making processes. This method operates without the need for reward labels, enabling better generalization across diverse domains and tasks. Our contributions are threefold: (1) We introduce a trajectory embedding approach that captures multiple abilities from state-action data. (2) The learned embeddings exhibit strong representational power across downstream tasks, including imitation, classification, clustering, and regression. (3) The embeddings demonstrate unique properties, such as controlling agent behaviors in IQ-Learn and an additive structure in the latent space. Experimental results confirm that our method outperforms traditional approaches, offering more flexible and powerful trajectory representations for various applications. Our code is available at \url{https://github.com/Erasmo1015/vte}.

\end{abstract}

\keywords{Representation Learning; Sequential Decision Making}

\newcommand{\BibTeX}{\rm B\kern-.05em{\sc i\kern-.025em b}\kern-.08em\TeX}

\begin{document}


\pagestyle{fancy}
\fancyhead{}


\maketitle

\def\thefootnote{*}\footnotetext{Equal Contribution.}

\section{Introduction}

Learning from state-action trajectories is a key requirement in sequential decision-making tasks, driving applications such as imitation, classification, regression, and clustering. For instance, autonomous vehicles need to mimic human driving behavior in scenarios like lane merging, while in robotics, trajectory learning is essential for replicating complex manipulation tasks. Traditional representation learning methods, though successful in static fields like computer vision (e.g., CLIP~\cite{radford2021learning}) and natural language processing (e.g., BERT~\cite{devlin2018bert}), often struggle to generalize to these sequential settings. In dynamic environments where trajectories unfold over time and reward signals may not always be present, the challenges of learning effective representations become more pronounced. This limitation raises a crucial question: 
\begin{center} \textit{How can we learn informative trajectory embeddings that capture the dynamic decision-making processes driving these trajectories?}
\end{center}

Although prior work, such as goal-conditioned learning~\cite{ajay2022conditional} has explored state-action trajectory representation, these methods rely on external labels such as goals or rewards, limiting their applicability across diverse domains. Most other works of representation learning in MDP~\cite{yang2021representation, parisi2022unsurprising, nair2022r3m, xiao2022masked, chen2021decision, liu2022masked, carroll2022uni, wu2023masked} focus on state representation learning, losing information of action sequences. A related topic is the work on skill (or options) extraction~\cite{hausman2018learning} from trajectories; however, skills or options capture information only about sub-trajectories. We find that a naive average of skills found in a trajectory does not provide informative embedding of the trajectory and thus cannot reach the return of the demonstrations
({\color{black}see ablation experiments in~\cref{sec:ablation}}).



In this work, we propose a novel approach that learns {\color{black}the informative embeddings of state-action trajectories}. Our approach has two stages. First, we use a skill extractor designed for sequential decision making inspired by~\citet{jiang2022learning}. We leverage Hierarchical State Space Models (HSSM) to extract a probability distribution of multiple possible skills from the trajectory. Next, this skill distribution is input into a shallow transformer and trained with a Variational autoencoder (VAE)~\cite{kingma2013auto} loss. This setup outputs a latent ability vector encapsulating the trajectory’s ability level. The process resembles a VAE, where the trajectory passes through a bottleneck, retaining key information for the decoder to reconstruct the trajectory.

Our approach provides several key advantages over previous methods. First, our approach does not require external labels, such as rewards or goal conditions, which are typically necessary in works like~\citet{zeng2024goal}. Second, our method encapsulates the entire trajectory's information without any extra labels such as task~\cite{hausman2018learning} or rewards~\cite{ajay2022conditional}. Third, our method can effectively learn from a dataset of trajectories generated by diverse policies, extracting latent representations that distinguish the trajectories generated by different policies in the latent space. Last but not least, our latent ability vectors enable a variety of downstream tasks including \emph{conditional imitation} learning that recovers the diverse policies that generated the dataset of trajectories, classification of trajectories, and regression tasks to predict the return.

\input{components/main_figure}

Our experimental results highlight several important characteristics of the latent ability vector. Our experiments demonstrate that previous baselines struggle to generate meaningful representations from the dataset of trajectories generated by diverse policies, resulting in poor recovery of the diverse policies that generated these trajectories. We show similar comparison for classification, clustering, and regression. Lastly, we demonstrate its strong representational power by showing that perturbing different dimensions of the trajectory embedding vector leads to distinct behavioral changes in agents.

To summarize, our main contributions include:

\begin{itemize}[nosep, itemsep=0.3em, left=5pt]
    \item \textbf{Unsupervised trajectory encoding.} We present a method that effectively extracts informative embeddings of state-action trajectories without any reward/goal labels. Our approach hinges on two key design choices: (1) learning trajectory embeddings through skill abstraction (representation of sub-trajectories), and (2) leveraging a transformer to capture the temporal nature of the skill sequence.
    \item \textbf{Diverse downstream tasks.} Our learned ability vector demonstrates strong representation across tasks like data generation (via imitation learning with ability embedding), classification, clustering, and regression.
    \item \textbf{Disentangled representation.} Our ability embedding shows interesting properties such as different dimensions of ability vectors controlling different behaviors of the agent in the conditional imitation learned policy, and an intuitive distance structure in the trajectory embedding space.
\end{itemize}

\section{Related work}
{\color{black}


\textbf{Representation Learning in MDP}\quad Representation learning in Markov Decision Processes (MDPs) is centered on extracting meaningful features from unlabelled trajectories to enhance performance in downstream tasks. Prior works primarily focus on either learning state representations~\cite{yang2021representation, parisi2022unsurprising, nair2022r3m, xiao2022masked, chen2021decision, liu2022masked, carroll2022uni, wu2023masked, laskin2020curl} or constructing world models~\cite{hafner2020mastering, hansen2022modem, seo2023masked, janner2021offline, ding2020mutual, nguyen2021temporal}. Recent advancements, such as GCPC~\cite{zeng2024goal}, expand the focus from sub-trajectory representations to embeddings in trajectory space. This work leverages sequence models like GPT and BERT to encode trajectories, with the resulting trajectory representations utilized to improve subsequent policy learning in offline reinforcement learning (RL) settings. However, these approaches still depend on reward (goal) labels, which differentiates them from our proposed method (which does not require reward/goal labels).

We develop a continuous state-action trajectory embedding in an unsupervised setting. This approach allows us to learn compact information representations that holistically capture the inherent policy behavior patterns present in the trajectories, in a broader context without the reliance on labeled data.

\noindent\textbf{Trajectory Embedding}\quad
There are several existing methods that embed trajectories. Grover et al. \cite{grover2018learning} proposed learning the trajectory embeddings using constrastive learning; however, their approach relies on labeled trajectories, whereas ours does not. Tangkaratt et al. \cite{tangkaratt2020variational} and related methods target learning the expert policies from the diverse-quality demonstrations but do not involve learning trajectory-level representations.
Other methods focus on encoding the state trajectory rather than encoding state-action trajectory. Hawke et al. \cite{hawke2020urban} obtains the embeddings through optical flows and some other sources to improve the model performance in autonomous driving. Gavenski et al. \cite{gavenski2024explorative} utilizes the path signatures to automatically encode the constraints.

\noindent\textbf{Imitation Learning}\quad 
Behavioral cloning (BC) is a basic offline imitation learning method that replicates expert actions without leveraging dynamics information. Advanced methods include GAIL~\cite{ho2016generative} and its variants~\cite{kostrikov2018discriminator, fu2017learning, baram2016model}, which optimize policies using a GAN-like~\cite{goodfellow2014generative} framework to learn from expert demonstrations in an online setting. In addition, offline imitation learning works ValueDICE~\cite{kostrikov2019imitation} and its variants~\cite{arenz2020non, jarrett2020strictly, chan2021scalable} focus on dynamics-aware approaches to minimize KL-divergence and integrate SAC updates. 

Imitation Learning (IL) and Inverse Reinforcement Learning (IRL) techniques often assume optimal demonstrations~\cite{abbeel2004apprenticeship, chen2020joint, gombolay2016apprenticeship}. The assumption often fails in many real-world scenarios. Kaiser et al. \cite{kaiser1995obtaining} analyze five sources of suboptimality, including unnecessary or incorrect actions and limited demonstration scenarios.

Thus, many works~\cite{burchfiel2016distance, brown2019extrapolating, brown2020better} have utilized sub-optimal demonstrations to learn an optimal policy. Some~\cite{burchfiel2016distance} utilizes ranking information among trajectories as a supervision signal, though this can be costly and error-prone. Other approaches include pre-labeling demonstrations as expert or non-expert~\cite{valko2013semi}, using crowd-sourced data with confidence scores~\cite{wu2019imitation}, and bootstraps from sub-optimal demonstrations to synthesize optimality-parameterized data~\cite{chen2021learning}.

However, a significant gap remains in imitating behaviors across different levels of optimality, or ``ability levels,'' whether sub-optimal or optimal. When dealing with mixtures of demonstrations of varying quality, the underlying behavior patterns that generated these demonstrations are often overlooked. Few works like Behavior Transformer (BeT)~\cite{shafiullah2022behavior} learn a multi-modal policy, allowing it to reconstruct different modes of behaviors. But the policy is parameterized by mixture of Gaussian rather than Gaussian prior, without an informative representation to represent the trajectory.
To address this, our approach reconstructs behavior patterns at different ability levels by learning continuous trajectory embeddings in an unsupervised manner. These embeddings, which capture the ability levels, facilitate various downstream tasks, including imitation of diverse policies, trajectory classification, and disentangled representation of behaviors.
}



\section{Preliminaries}

{\color{black}
\textbf{Problem setting}\quad We consider environments represented as a Markov decision process (MDP), which is defined by a tuple ($\mathcal{X}, \mathcal{A}, p_0, {P}, r, \gamma)$.
$\mathcal{X}, \mathcal{A}$ denote state and action spaces, $p_0$ and ${P}(\vx'|\vx, \va)$ represent the initial state distribution and the dynamics. {\color{black} The reward function is $r(\vx, \va) \in \mathbb{R}$, and $\gamma \in (0, 1)$ is the discount factor.} $\Pi$ denotes the set of all stationary stochastic policies that map states in $\mathcal{X}$ to actions in $\mathcal{A}$. 

We assume access to an offline dataset, $\mathcal{D}=\{\tau_i\}_{i=1}^N$, where each trajectory consists of a sequence of states $\vx\in \mathcal{X}$ and actions $\va\in\mathcal{A}$: $\tau_i=\{(\vx_0,\va_0), (\vx_1,\va_1), (\vx_2,\va_2), ...\}$. The trajectories are assumed to be generated by a policy that is conditional on an ability level, $\ve$: 
\begin{align}\label{eq:gen}
    \vx_0\sim p_0, \ve\sim p(\ve), \va_t\sim\pi(\va_t|\vx_t, \ve), \vx_{t+1}\sim P(\vx_{t+1}|\va_t,\vx_t)
\end{align}
$\ve$ is unobserved and it  determines the specific policy $\pi(\cdot|\vx,\ve)$ used to generate the trajectory. Our goal is to learn this latent variable, $\ve_{\tau}$, also referred to as trajectory embedding for any given trajectory, $\tau$. We utilize an offline dataset $\mathcal{D}$ of trajectories to train our mechanisms. We further show that such informative representations of the trajectory enable various downstream tasks, including policy imitation $\pi(\cdot | \vx, \ve)$, trajectory classification, and regression on the trajectory’s return. 

}




\noindent\textbf{Variational Autoencoders~\cite{kingma2013auto}}\quad 
{\color{black}
Our approach that learns the latent representation of the trajectory is inspired by the well known Variational AutoEncoder (VAE) framework. VAE has an encoder-decoder architecture, where the encoder (a neural network with weights $\phi$) learns a probability distribution on a latent space given an input data point. During training, the decoder (a neural network with weights $\theta$) is used to reconstruct the input given the latent space encoding.
The learning happens by maximizing the evidence lower bound (ELBO) of the intractable log-likelihood $\log p(\tau)$:
\begin{align}\label{eq:vae}
\log p(\tau) \geq -D_{KL}(q_\phi(\ve | \tau) \| p(\ve)) + \mathbb{E}_{q_\phi(\ve \mid \tau)}[\log p_\theta(\tau | \ve)] \textrm{,}
\end{align}
\noindent where $\ve$ is a sample in the latent space from the approximate posterior distribution $q_\phi(\ve | \tau)$. In the first term, the prior $p$ is chosen as standard Normal distribution. The second term corresponds to a reconstruction of an observed sample generated with the likelihood $ p(\tau | \ve) $. A common choice for the approximate posterior $ q $ is a Gaussian distribution, $ \mathcal{N}(\boldsymbol{\mu}, \Sigma) $, where $ \boldsymbol{\mu} $ and $ \Sigma $ are outputs of the encoder $q_\phi$ network (as described in~\cite{kingma2013auto}). 
Once trained, we can draw samples in the latent space and the decoder can generate samples in the space of observations. 

}

\noindent\textbf{Skill extraction via compression~\cite{jiang2022learning}}\quad 
{\color{black}
Skill, or options learning, derives higher-level abstractions from state-action sequences, which can help compress the entire sequence. The Learning Options via Compression (LOVE) approach \cite{jiang2022learning} has proven effective by modeling state-action sequences as a generative process that relies on specific latent (unobserved) variables at each time step.

In this process, the latent skill variable $\vz \in \mathcal{Z}$ represents the skill used at a given time step, where $\mathcal{Z}$
 is the set of possible skills. Another latent variable, $\vm \in \{0,1\}$ indicates whether a new skill starts (1) or the current skill continues (0). At each step,  $\vm$ and $\vz$ influence the hidden state $\vs$, which in turn affects the observed state $\vx$. Over time, $\vz$ is influenced by the current $\vm$ and the previous skill, while $\vm$ depends on the previous state. For a detailed explanation of the graphical model, we refer readers to the LOVE paper.

Mathematically, the generative process for the action $\va_{1:T}$ conditional on the observation $\vx_{1:T}$ is: 
\begin{align}\label{eq.graphic_model}
&p\left(\vz_{1:T}, \vs_{1:T}, \vm_{1:T}, \va_{1:T} \mid \vx_{1:T}\right) =   \\
&\hspace{0.2in} \prod_{t=1}^T p\left(\va_t \mid \vs_t\right) p\left(\vm_t \mid \vs_{t-1}\right) p\left(\vs_t \mid \vx_t, \vz_t\right) p\left(\vz_t \mid \vx_t, \vz_{t-1}, \vm_{t-1}\right) \nonumber
\end{align}

The skill learning in LOVE is achieved by maximizing the likelihood of the sequences while penalizing the description length of the skills. 
Due to intractability of the above likelihood, the authors introduce a variational distribution: 
\begin{align}\label{eq:var_dist}
& q_{\phi}  \left(\vz_{1:T}, \vs_{1:T}, \vm_{1:T} \mid \vx_{1:T}, \va_{1:T}\right) = 
  \\
&\quad \prod_{t=1}^T q_{\phi}\left(\vm_t \mid \vx_{1:t}\right) q_{\phi}\left(\vz_t \mid \vz_{t-1}, \vm_t, \vx_{1:T}, \va_{1:T}\right)  q_{\phi}\left(\vs_t \mid \vz_t, \vx_t\right) \nonumber
\end{align}

\noindent Overall, this yields a model with 3 learned components: 1) A \emph{state abstraction posterior} $q_{\phi}(\vs_t | \vz_t, \vx_t)$;
2) A \emph{termination policy} $q_{\phi}(\vm_t | \vx_{1:t})$ that decides if the previous skill ends; 3) A \emph{skill posterior} $q_{\phi}(\vz_t | \vz_{t - 1}, \vm_t, \vx_{1:T}, \va_{1:T})$ to determine the current skill $\vz_t$. 

LOVE penalizes the description length of the skills to improve the quality of the learned skills. The description length of the skills is measured by: 
\begin{equation*}
    \textsc{InfoCost}(\phi; p_{\vz}) = -\E_{\substack{\tau_{1:T},\\\vm_{1:T},\\\vz_{1:T}}}\left[\sum_{t=1}^T \log p_{\vz}(\vz_t) \vm_t\right]\textrm{.}
\end{equation*}

\noindent Combining \textsc{InfoCost} with maximal likelihood objective, the authors propose to solve the following optimization problem: 

\begin{align}\label{eq:love_obj}
    \min _{\phi, p_\vz} \textsc{InfoCost}(\phi; p_{\vz})\quad \text { s.t. } \mathcal{L}_{\mathrm{ELBO}}(\phi) \leq C
\textrm{,}
\end{align}

\noindent where $\mathcal{L}_{\mathrm{ELBO}}(\phi)$ is the negated evidence lower bound of the likelihood defined by {\color{black}\cref{eq.graphic_model}} (detailed description of $\mathcal{L}_{\mathrm{ELBO}}(\phi)$ can be found in Appendix B of~\citet{jiang2022learning}). Once solved, one can infer the skill variables $\vz_{1:T}$ and boundary variables $\vm_{1:T}$ given the trajectory of observations $\boldsymbol{x}_{1:T}$. $\vz_{1:T}$, alone with $\vm_{1:T}$, provides all the information about learned skills.

}

\noindent\textbf{Imitation Learning}\quad
{\color{black} 
Imitation learning (IL) aims to learn a policy for performing a task based solely on expert demonstrations, which consist of state-action trajectories without any reinforcement signals~\cite{ho2016generative}. IQ-Learn~\cite{garg2021iq} has been proposed as an efficient and robust imitation learning algorithm that learns a single Q-function, implicitly capturing both reward and policy. The theoretical framework of IQ-Learn builds on the classic inverse reinforcement learning (IRL) objective: 

\begin{equation*}
\max_{r \in \mathcal{R}} \min_{\pi \in \Pi} L(\pi, r) = \mathbb{E}_{\rho_E}[r(\vs, \va)] - \mathbb{E}_{\rho_\pi}[r(\vs, \va)] - H(\pi) - \psi(r)\textrm{,}
\end{equation*}

\noindent where $\rho_E$ and $\rho_\pi$ denote the occupancy measures of the expert policy and the learned policy, respectively, $r$ represents a learnable reward function, $H(\pi)$ refers to the entropy of policy $\pi$, and $\psi$ is a convex reward regularizer. 

The authors showed that this objective can be achieved by only maximizing the $Q$ function in the following objective: 

\begin{align*}
\mathcal{J}(\pi, Q) &= \mathbb{E}_{\rho_E}\left[f\left(Q - \gamma \mathbb{E}_{\vs^{\prime} \sim {P}(\cdot \mid \vs, \va)} V^\pi\left(\vs^{\prime}\right)\right)\right] 
\\ &\quad 
-(1-\gamma) \mathbb{E}_{\vs_0\sim\rho_0}\left[V^\pi\left(\vs_0\right)\right]\textrm{,}
\end{align*}
\begin{align*}
    V^\pi(\vs) = \mathbb{E}_{\va \sim \pi(\cdot | \vs)} \left[ Q(\vs, \va) - \log \pi(\va | \vs) \right]\textrm{,}
\end{align*}

\noindent where $f$ is a concave function associated with the choice of the reward regularizer $\psi$, and $\rho_0$ represents the initial state distribution. For a fixed $Q$, the soft actor-critic (SAC)~\cite{haarnoja2018soft} update: $\max_\pi \mathbb{E}_{\vs \sim \beta, \va \sim \pi(\cdot | \vs)}[Q(\vs, \va)-\log \pi(\va | \vs)]$, brings $\pi$ closer to $\pi_Q$. Here, $\beta$ refers to the distribution of previously sampled states. 

Although imitation learning algorithms such as IQ-Learn are typically used to learn a policy from the expert demonstrations, in~\cref{sec:method}, we discuss how these methods can, in principle, learn from demonstrations of varying quality, including non-expert data. Empirically, we show that IQ-Learn can be applied to a mixture of expert and non-expert data using the {\color{black} trajectory embedding} inferred by our approach. 
}




\section{Methodology}\label{sec:method}
Our approach has two stages, as illustrated in Figure~\ref{fig:main}: Section~\ref{sec:love} covers skill extraction and Section~\ref{sec:vte} explains transformer usage and VAE style learning.

\subsection{Learning Skills via Compression}
\label{sec:love}

Our main idea is based on the assumption that skills (which are temporal abstractions of the trajectory) naturally combine detailed, step-by-step information and capture the patterns within segments of the trajectory. This suggests that learning an embedding for the entire trajectory by focusing on the skill space (i.e., the space of these higher-level skills) would be easier than trying to learn it directly from the raw state-action space. To achieve this, we utilize the skill learning technique called LOVE, as introduced in the work by \citet{jiang2022learning}.

By solving the optimization problem of~\cref{eq:love_obj}, we can readily acquire the latent variable values, $\vz_{1:T}$, skill change variable values, $\vm_{1:T}$, the learned termination policy, $q_{\phi}(\vm_t | \vx_{1:t})$, and skill posterior, $q_{\phi}(\vz_t | \vz_{t - 1}, \vm_t, \vx_{1:T}, \va_{1:T})$ (refer to~\cref{eq:var_dist}). For brevity, we denote this process by $\vz_{1:T}, \vm_{1:T}=\textrm{SE}_\phi(\vx_{1:T}, \va_{1:T})$, where $\textrm{SE}_\phi$ is named skill extractor. 

However, we noticed that the skills we sampled were not providing enough useful information to recover the trajectories. To address this and make the most of the skill knowledge without losing any detail, we capture the full information about the distributions of $\vz_t$(the skill) and $\vm_t$ (the indicator for whether a new skill starts). In the LOVE model, $\vz_t$ is sampled from a categorical distribution, and $\vm_t$ is sampled from a Bernoulli distribution.

Instead of just working with these samples, we aim to capture more detail by using the logit vectors that describe the underlying distributions. Specifically, we represent $\vz_t$ as a vector in $\mathbb{R}^l$ (real-valued space of length l) and  $\vm_t$  as a vector in $\mathbb{R}^2$. This approach, with a slight abuse of notation, allows us to work with the full distributions rather than just the samples. We denote this process as $\vz_{1:T}, \vm_{1:T}=\textrm{SE-Logit}_\phi(\vx_{1:T}, \va_{1:T})$, meaning the logits are derived from the input state and action sequences.


\subsection{Variational Trajectory Encoding} \label{sec:vte}
One naive way of producing trajectory embeddings is to apply the mean pooling operation on $\vz_{1:T}$:
\begin{align}\label{eq:pooling}
    \ve = \frac{1}{T}\sum_{i=1}^T \vz_i\textrm{.}
\end{align}
\noindent However, we find this embedding cannot capture information at the trajectory level as it only contains information about the skills that are present in the trajectory. Thus, we propose to learn a more informative embedding by combining not just skills but also the ordering of those skills in the trajectory. We employ a VAE to compute this embedding.  Starting from~\cref{eq:vae}, we show how to construct the required terms in the VAE loss. We first describe how to obtain the encoding part of VAE, i.e., computing $q_\phi(\ve|\tau)$, which is represented using a normal distribution, $ \mathcal{N}(\boldsymbol{\mu}, \Sigma) $.

{\color{black}
\begin{itemize}[nosep, itemsep=0.3em, left=5pt]
    \item Obtain $\vz_{1:T}, \vm_{1:T}=\textrm{SE-Logit}_\phi(\vx_{1:T}, \va_{1:T})$
    \item Perform $\ve^z_{1:T}=\textrm{MLP}(\vz_{1:T})$ and $\ve^m_{1:T}=\textrm{MLP}(\vm_{1:T})$. The multi-layer perceptron (MLP) converts $\vz$ and $\vm$ to the same size.
    \item Concatenate for each time step, obtaining $\ve^{z,m}_{1:T}$, where $\ve^{z,m}_{i}=\textrm{Concat}(\ve^z_i, \ve^m_i)$
    \item Process $\ve^{z,m}_{1:T}$ via a shallow transformer, and obtain a mean pooling of its output.
    \item Obtain mean $ \boldsymbol{\mu} $ and standard deviation $ \Sigma $ by mapping the transformer's output via a fully connected neural layer. 
\end{itemize}

Next, we describe the decoding part and how to optimize the reconstruction loss term $\mathbb{E}_{q_\phi(\ve | \tau)}[\log p_\theta(\tau | \ve)]$. As a trajectory is a sequence of states and actions, we obtain the following result:
\begin{proposition} For any given environment, 
    ${\arg\max}_\theta \log p_\theta(\tau | \ve) = {\arg\max}_\theta \sum_t \log p_\theta(\va_t|\vx_t, \ve)$.
\end{proposition}
\begin{proof}
    The proof readily follows from the fact that $ p_\theta(\tau | \ve) = p(\vx_0) p_\theta(\va_0 | \vx_0 ,\ve) p(\vx_1|\vx_0, \va_0) p_\theta(\va_1 | \vx_1, \ve)\ldots$. Applying the log, the product becomes a sum and only the $p_\theta(\va_t|\vx_t, \ve)$ terms depend on $\theta$.
\end{proof}
With this result above, optimizing the reconstruction loss is exactly the same as behavioral cloning~\cite{pomerleau1988alvinn}. Thus, we use a neural network to parameterize the policy $p_\theta(a_t|x_t, \ve)$, and optimize $\sum_t \mathbb{E}_{q_\phi(\ve | \tau)}[\log p_\theta(a_t|x_t, \ve)]$ for both $\phi$ and $\theta$. This term is optimized for $\phi$ using the standard reparameterization trick from VAEs.

}

We refer to our approach as Variational Trajectory Encoding (VTE). Compared with the embedding from the mean pooling of the skill embeddings in~\cref{eq:pooling}, our approach shows better performance (see~\cref{sec:ablation}). Furthermore, we analyze the differences between trajectory embeddings produced by mean pooling and our method in the Appendix.



\subsection{Trajectory Embedding for Downstream Tasks}
In this section, we utilize the trajectory embedding obtained from VTE to address different downstream tasks: (1) Imitating trajectories from different expertise level policies (e.g., Expert, Good, Bad); (2) Classifying expertise level of the trajectory; and (3) Predicting return from a trajectory, without a reward label. 

\subsubsection{Imitating Trajectories of Varying Abilities}\label{sec:imit-ability}
In this task, we consider an offline dataset $\mathcal{D}$ consisting of  trajectories collected from a mixture of expert and non-expert policies. Recall the trajectory generation process described by~\cref{eq:gen}, where $\ve$ captures the intrinsic \emph{ability level} of the policy that generated this trajectory. We define this ability level more concretely (in Section~\ref{sec:experimentsetup}) as a range of returns of the policy that generated this trajectory. Our goal is to learn these policies $\pi(\cdot|\vx,\ve)$ that generate the trajectories.

To learn $\pi(\cdot|\vx,\ve)$, we utilize the recent IQ-Learn framework~\cite{garg2021iq} and modify it to a conditional version. The IQ-Learn approach maintains an actor $\pi$ and critic $Q$ neural network. We make the conditional version by introducing conditions into the actor and critic neural network. We have conditional actor, $\pi(\cdot|\vx, \ve)$, and critic, $Q(\vx,\va|\ve)$. Details are provided in~\cref{alg:iq} in the Appendix.

\subsubsection{Classification and Regression}
For classification task, we train a classifier based on multi-layer perceptron (MLP) to predict the ability level with the same dataset used in imitation learning task (see~\cref{sec:imit-ability}), i.e. we learn a mapping $f_{\textrm{CLS}}:\mathbb{R}^d\rightarrow \{1, ..., M\}$ from the trajectory embedding to its ability level, where $d$ is the dimension of the trajectory embedding and $M$ is the number of ability levels.
For regression task, we train an MLP $f_{\textrm{REG}}: \mathbb{R}^d\rightarrow \mathbb{R}$ to predict the return of the trajectory. 

Note that the labels (ability level and return) are inaccessible to our trajectory encoding algorithm; we include them here solely to evaluate embedding quality.
The classifier achieves high accuracy only if our method learns high-quality trajectory embeddings, but it will fail when the embeddings are of low quality. In the worst scenario, all embeddings are identical, and the classifier has to make random guesses. The same rationale applies to the regression task.

\section{Experiments}

In this section, we answer the following questions through experiments: (1) Is the learned trajectory embedding well-structured? (2) How effective is the trajectory embedding on downstream tasks? (3) Does the trajectory embedding exhibit specific properties? (4) How does a simple mean polling of skill embedding perform?  

We answer question (1) in~\cref{sec:structure} and question (2) in~\cref{sec:imitating,sec:cls}. Question (3) is extensively analyzed in~\cref{sec:property}, and finally an ablation experiment is done in Section~\ref{sec:ablation} for answering question (4).

\subsection{Experiment Setup} \label{sec:experimentsetup}

All experiments were run on NVIDIA Quadro RTX 6000 GPUs, CUDA 11.0 with Python version 3.7.12
in Pytorch 1.8.0.
Hyperparameter settings are in the appendix.

\noindent\textbf{Ability Level}\quad As mentioned in~\cref{sec:imit-ability}, our downstream tasks are trained on a dataset with varying ability levels. We define the ability level of a policy by the average return of its collected trajectories. For instance, a low-ability policy generates trajectories with returns being $(400\pm 100)$, while an expert-ability policy generates trajectories with returns being $(2000\pm 100)$. 

\noindent\textbf{Dataset}\quad Due to lack of a public dataset for our tasks, we generated the dataset ourselves. We selected three environments from MuJoCo~\cite{todorov2012mujoco}, \texttt{Hopper}, \texttt{Walker2D}, and \texttt{Half-Cheetah}. For each environment, we trained an RL agent using Soft Actor Critic (SAC) to the expert level, and saved checkpoints throughout the training. We then took three checkpoints to generate trajectories, corresponding to \textit{low}, \textit{medium}, and \textit{expert} ability levels, with 300 trajectories generated per ability level. The return information is provided in~\cref{tab:rewards}.


\begin{figure*}[t]
    \centering
    \hspace*{-0.05\linewidth}  
    \includegraphics[width=1.1\linewidth]{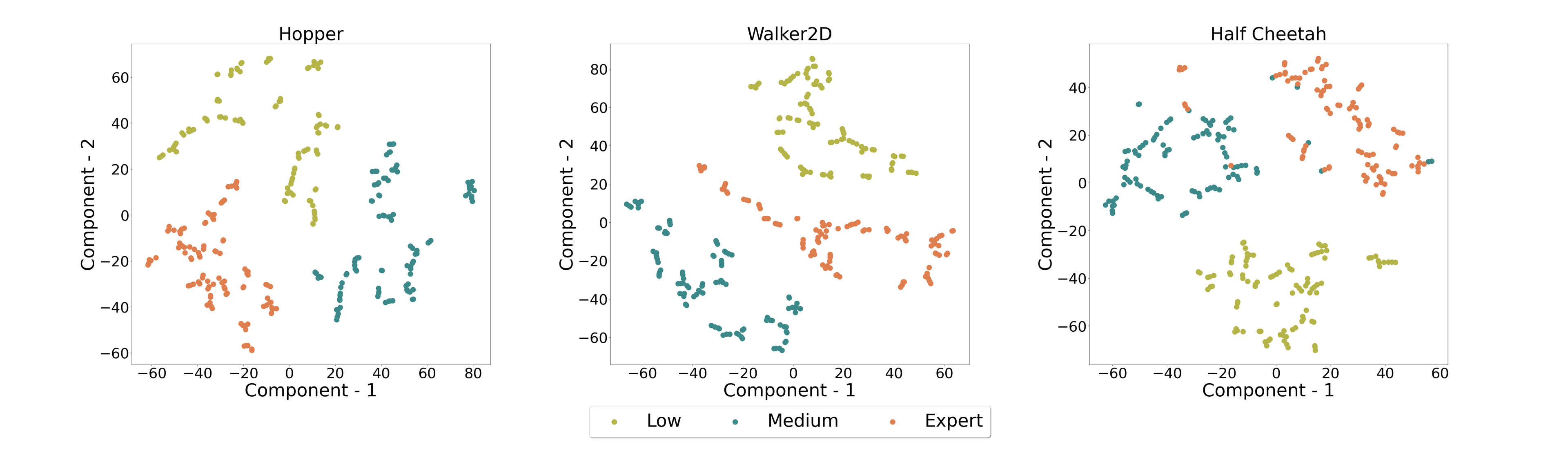}  
    \caption{tSNE Clustering Analysis}
    \label{fig:tSNE}
\end{figure*}

\begin{table}[t]
    \centering
    \caption{Dataset returns across different ability levels for each environment}
    \label{tab:rewards}
    \begin{tabular}{ c c c c }
        \toprule
        \textbf{Environment} & \textbf{Low} & \textbf{Medium} & \textbf{Expert} \\
        \midrule
        \textbf{Hopper} & $409.0 {\text{\color{gray} $\pm 4.9$}}$ & $885.3 {\text{\color{gray} $\pm 49.2$}}$ & $3206.8 {\text{\color{gray} $\pm 18.1$}}$ \\
        \textbf{Walker2D} & $3064.2 {\text{\color{gray} $\pm 43.3$}}$ & $4359.6 {\text{\color{gray} $\pm 51.1$}}$ & $5912.2 {\text{\color{gray} $\pm 25.3$}}$ \\
        \textbf{Half Cheetah} & $2402.9 {\text{\color{gray} $\pm 21.2$}}$ & $4197.6 {\text{\color{gray} $\pm 54.0$}}$ & $6321.9 {\text{\color{gray} $\pm 61.3$}}$ \\
        \bottomrule
    \end{tabular}
\end{table}

\begin{table*}[t]
    \centering
    \caption{Relative L2 norm difference between the learned Returns and the target Returns. }
    
    \begin{tabular}{ c c c c c c}
        \toprule
        \textbf{Environment}  & \textbf{Known-Abl} & \textbf{VTE-MLP} & \textbf{GCPC} & \textbf{GCPC-NR} & \textbf{VTE} \\
        \midrule
        Hopper    & $\underline{9.2} {\text{\color{gray} $\pm 16.7$}}$  & $24.6 {\text{\color{gray} $\pm 20.5$}}$ & $21.3 {\text{\color{gray} $\pm 25.9$}}$ & $27.6 {\text{\color{gray} $\pm 22.5$}}$ & $\textbf{0.9{\text{\color{gray} $\pm 2.8$}}}$ \\
        Walker2D    & $\textbf{7.7{\text{\color{gray}$\pm 14.8$}}}$ & $45.3 {\text{\color{gray} $\pm 23.1$}}$ & $27.8 {\text{\color{gray} $\pm 18.3$}}$ & $46.2 {\text{\color{gray} $\pm 23.4$}}$ & $\underline{13.1}{\text{\color{gray} $\pm 20.3$}}$ \\
        Half Cheetah   & $\textbf{1.3{\text{\color{gray} $\pm 1.1$}}}$ & $19.1 {\text{\color{gray} $\pm 7.0$}}$ & $28.7 {\text{\color{gray} $\pm 20.0$}}$ &  $47.3 {\text{\color{gray} $\pm 24.9$}}$ & $\underline{1.5} {\text{\color{gray} $\pm 1.2$}}$ \\
        \bottomrule
    \end{tabular}
    \label{tab:l2}
\end{table*}

\noindent\textbf{Baselines}\quad GCPC~\cite{zeng2024goal} is the closest work to ours, which utilizes the encoder-decoder transformer to encode trajectories. However, it is designed for offline RL and thus relies on rewards in the dataset. We evaluate two versions of GCPC: the original, which includes rewards, referred to as \textbf{GCPC}, and an adapted version that removes rewards from the input, referred to as \textbf{GCPC-NR} (No Reward). 
We also modify our \ours{} framework by replacing the skill extractor $\textrm{SE-Logit}$ with an MLP, referred to as \textbf{\ours-MLP} (see~\cref{sec:method}). Additionally, we introduce a strong baseline where the downstream task has access to the ability levels. For the imitation learning task, we consider this baseline an upper bound, which we call \textbf{Known-Abl} (Known Ability).

{\color{black} 
\noindent\textbf{Evaluation Metrics}\quad For classification, we use accuracy as the performance metric. For other downstream tasks, we measure performance by calculating the relative L2 norm of the difference between the learned and dataset returns. This is done by averaging the L2 norm difference for each learned return and its corresponding target return, and then normalizing by the target return.
}
    

\subsection{\ours{} Generates Well-Structured Embeddings}
\label{sec:structure}

We perform clustering analysis using tSNE \cite{van2008visualizing} on the latent trajectory embeddings to verify whether different ability levels can be distinguished. As shown in Figure~\ref{fig:tSNE}, three distinct clusters emerge in each environment, corresponding to the low, medium, and expert ability levels, confirming that our method successfully learns the latent vectors that separate ability levels. We also provide the clustering results using Principal Component Analysis (PCA) in the appendix.

\subsection{\ours{} Enables Ability-Conditioned Imitating}
\label{sec:imitating}

We assess the imitation performance of our method across three environments: \texttt{Hopper}, \texttt{Walker2D}, and \texttt{Half-Cheetah}, each with three skill levels: low, medium, and expert. 
We extract the latent trajectory embeddings using our approach and then learn a policy conditioned on these vectors using conditional IQ-Learn. 
The objective is to match the target return of that trajectory as observed in data, rather than learning the optimal policy.

\cref{tab:l2} summarizes the results against the baselines using the relative L2 norm loss between the average evaluation return of the policy conditioned on a trajectory embedding and the return of that trajectory in the dataset, where a lower value is better. {\color{black} We have bolded the best results and underlined the second-best results for each environment, with all results expressed as percentages. Our method consistently outperforms both baselines in the imitation task, achieving performance comparable to the upper bound set by Known-Abl. Despite the fact that Known-Abl and GCPC utilize additional reward information, our method surpasses GCPC across all environments and demonstrates lower error rates than Known-Abl in the Hopper environment. } This demonstrates that our trajectory embedding significantly enhances imitation learning when the dataset consists of mixed ability levels.

{\color{black}
Figure~\ref{fig:hopper_train} also presents the evaluation of returns on Hopper. For each ability level, our method facilitates learning conditioned on the latent trajectory embeddings, achieving results close to the upper-bound of the Known-Abl. In contrast, the VTE-MLP shows instability during training.
}


\begin{figure*}[t]
    \centering
    \includegraphics[width=1.\linewidth]{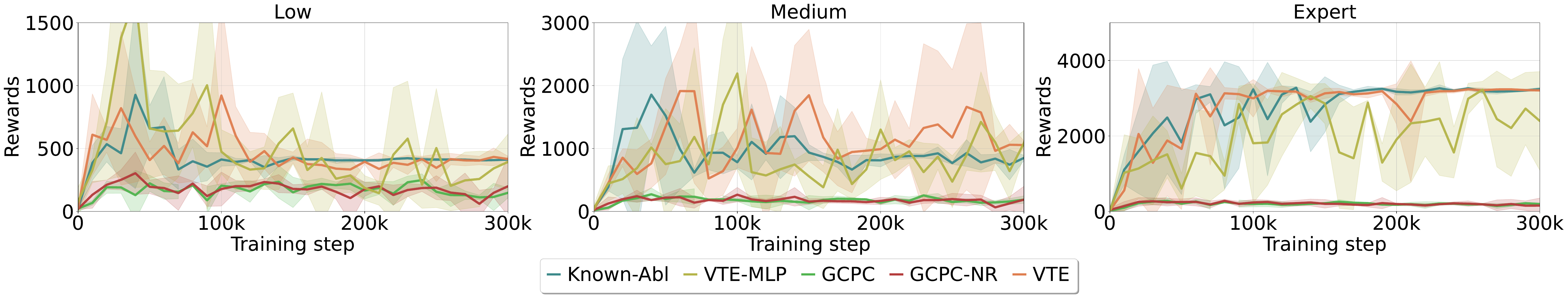}  
    \caption{Evaluation curve of returns on Hopper for different ability levels.}
    \label{fig:hopper_train}
\end{figure*}

\begin{figure*}[t]
    \centering
    \includegraphics[width=\linewidth]{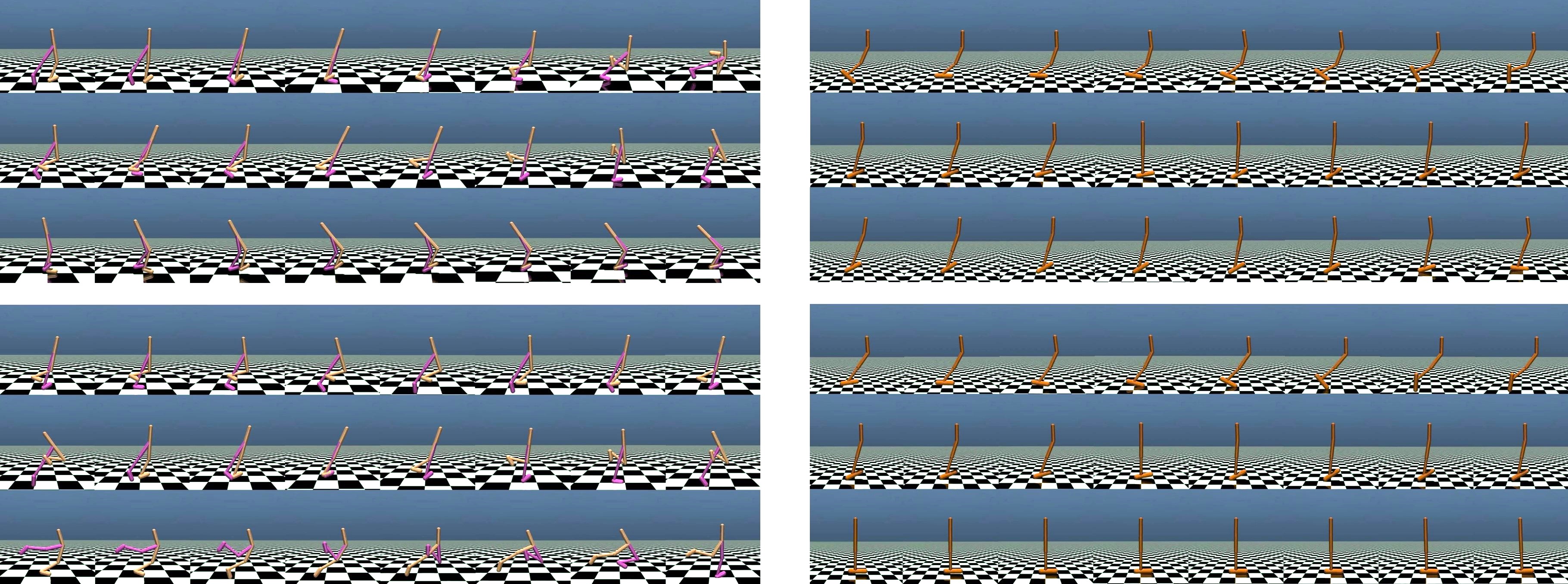}
    \caption{Overall visual comparison of change in behavior in \texttt{Walker2D} environment, presented in the left column, and \texttt{Hopper} environment presented in the right column, when a  dimension of the trajectory embedding is changed. Each of the four boxes is a perturbation result on one of the 10 vector dimensions of the trajectory embedding. Inside each box there are three different rows showing a sequence of frames. \textbf{First Row:} The value of the dimension is decreased. \textbf{Second Row:} The value is not perturbed as a control group. \textbf{Third Row:} The value of the dimension is increased (see more description in text).}
    \label{fig:video}
\end{figure*}

\subsection{\ours{} Facilitates Trajectory Classification and Regression}
\label{sec:cls}

{\color{black} We present the classification results in Table \ref{tab:classification} and the regression of rewards in Table \ref{tab:regression}. It is evident that VTE-MLP and VTE achieves 100\% classification accuracy within 80 epochs. For the regression task, the different methods appear to achieve similar results.}


\begin{table}[t]
    \centering
    \caption{Classification Accuracy on Ability Levels(\%).}
    \begin{tabular}{ c c c c }
        \toprule
        \textbf{Environment} & \textbf{GCPC-NR} & \textbf{VTE-MLP} & \textbf{VTE} \\
        \midrule
        Hopper & 34.2 & 100.0 & 100.0 \\
        Walker2D & 34.2 & 100.0 & 100.0 \\
        Half Cheetah & 32.9 & 100.0 & 100.0\\
        \bottomrule
    \end{tabular}
    \label{tab:classification}
\end{table}

\begin{table}[t]
    \centering
    \caption{Relative Regression Error on Rewards(\%).}
    \begin{tabular}{ c c c c }
        \toprule
        \textbf{Environment} & \textbf{GCPC-NR} & \textbf{VTE-MLP} & \textbf{VTE} \\
        \midrule
        Hopper & 3.8 & 2.7 & 3.1 \\
        Walker2D & 0.7 & 0.8 & 0.7 \\
        Half Cheetah & 0.9 & 0.7 & 0.8\\
        \bottomrule
    \end{tabular}
    \label{tab:regression}
\end{table}

\subsection{Property of Trajectory Embedding}
\label{sec:property}

\paragraph{Perturbed Conditions}
Figure~\ref{fig:video} visually compares behaviors in the \texttt{Walker2D} and \texttt{Hopper} environments, highlighting the distinct behaviors generated when specific dimensions of the trajectory embeddings are perturbed in opposing directions. Perturbing different dimensions produces varied behaviors. The figure highlights results for two dimensions in each environment. The left and right columns display results from the \texttt{Walker2D} and \texttt{Hopper} environments. In each box, we present the results of perturbing one dimension: the middle row represents the control group without perturbation, while the first and third rows show the effects of decreasing and increasing the value of that dimension, respectively. 

\begin{figure*}[t]
    \centering
    \hspace*{-0.03\linewidth}  
    \includegraphics[width=1.05\linewidth]{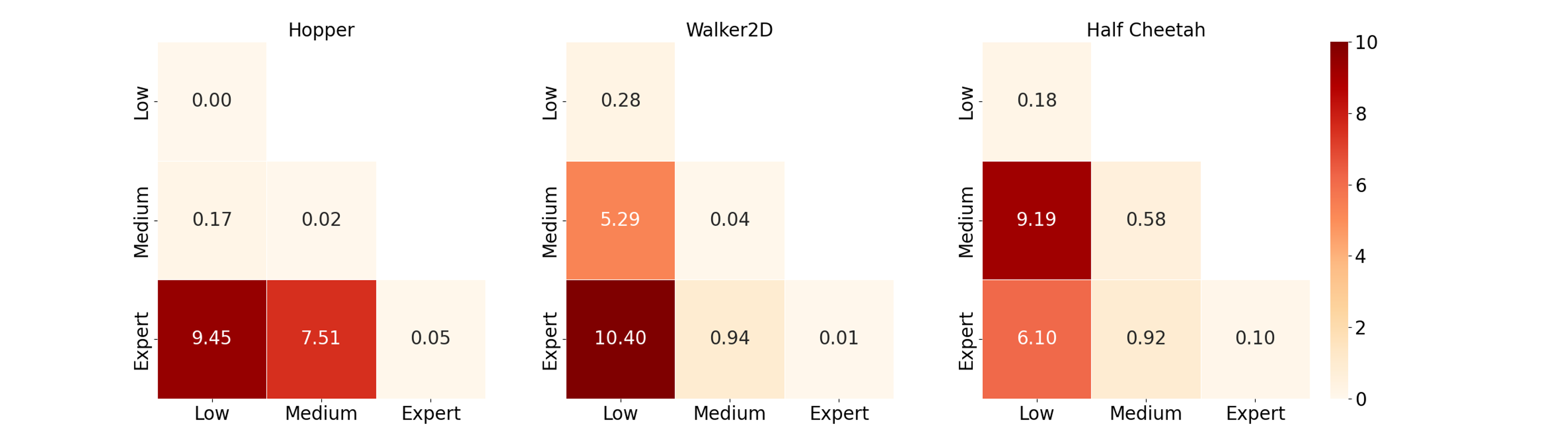}  
    \caption{Heatmap of Wasserstein distance (see definition in text) between distribution of trajectory embeddings for different ability levels.}
    \label{fig:distance}
\end{figure*}

In the top-left box, the \texttt{Walker2D} agent leans forward during walking when the value decreases and falls backward when it increases. In the bottom-left box, the agent walks leisurely with a decreased value but sprints with larger body movements when the value is increased. In the top-right box, the \texttt{Hopper} agent bends its leg with a reduced value but keeps it rigid when the value is increased. In the bottom-right box, the agent eagerly hops forward with a decreased value but remains still with an increased value.

These results clearly demonstrate that perturbing a single dimension in the latent trajectory embedding produces contrasting behaviors, illustrating the high degree of disentanglement in our learned representations.

\paragraph{Measure Distance Between Polices}
{\color{black}
We compute Wasserstein distances between policies at different ability levels, as shown in Figure~\ref{fig:distance}. The Wasserstein distance quantifies the minimum "cost" to transform one probability distribution into another. For each environment, we collect 10 trajectory embeddings to form an empirical distribution, representing the policy's ability level. We then calculate the Wasserstein distance between the distributions from different ability levels. For each policy, we also compare distributions within the same corresponding ability level to assess intra-level variability. The results demonstrate that distances within the same ability level are relatively small, while most distances between different ability levels are significantly larger. This indicates that the learned trajectory embeddings effectively distinguish between policies of different ability levels.
}

\begin{table}[t]
    \centering
    \caption{Range of the evaluation returns on Walker2D shown across ability levels for the returns seen in dataset, and for the conditional policies learned from embeddings of mean skill pooling and our VTE approach.}
    \label{tab:ablation_skill_pooling}
    \begin{tabular}{ c c c c }
        \toprule
        \textbf{Method} & \textbf{Low} & \textbf{Medium} & \textbf{Expert} \\
        \midrule
        \textbf{In Dataset} & $3064.2 {\text{\color{gray} $\pm 43.3$}}$ & $4359.6 {\text{\color{gray} $\pm 51.1$}}$ & $5912.2 {\text{\color{gray} $\pm 25.3$}}$ \\
        \textbf{Skill Pooling} & $1058.3 {\text{\color{gray} $\pm 752.3$}}$ & $776.0 {\text{\color{gray} $\pm 438.5$}}$ & $796.9 {\text{\color{gray} $\pm 561.7$}}$ \\
        \textbf{VTE} & $2045.9 {\text{\color{gray} $\pm 1101.9$}}$ & $3782.9 {\text{\color{gray} $\pm 1043.5$}}$ & $5053.0 {\text{\color{gray} $\pm 1834.4$}}$ \\
        \bottomrule
    \end{tabular}
\end{table}

\subsection{Ablation Experiment}\label{sec:ablation}
As mentioned in Section~\ref{sec:vte}, just a mean pooling of skills does not produce desired results. Here we present a result that shows the variation in returns (from imitating) based on mean pooled skill embedding and our VTE embeddings. The result in Table~\ref{tab:ablation_skill_pooling} on \textsc{Walker2D} show that the mean pooled skill embedding shows less variation in returns (across ability levels), thereby not able to learn the returns of low, medium, and expert ability levels.

\section{Conclusion}

In conclusion, this work introduces a novel unsupervised approach for encoding state-action trajectories into informative embeddings without the need for external reward or goal labels. The method leverages hierarchical skill abstraction and a transformer and VAE-based architecture to capture the temporal dynamics of trajectory skills. The resulting informative trajectory embedding demonstrates strong representation capabilities across various downstream tasks, including imitation learning, classification, clustering, and regression. Moreover, the disentangled nature of the learned embedding allows for intuitive control of agent behaviors and meaningful differentiation in the trajectory embedding space.

\section{Acknowledgement}
This research/project is supported by the National Research Foundation Singapore and DSO National Laboratories under the AI Singapore Programme (AISG Award No: AISG2-RP-2020-017) and the grant W911NF-24-1-0038 from the US Army Research Office.






\bibliographystyle{ACM-Reference-Format} 
\bibliography{reference}

\clearpage
\newpage
\appendix
\section{Conditional IQ-Learn Algorithm}\label{app:iq}


\begin{algorithm}
\caption{Conditional Inverse soft Q-Learning (adapted from~\citet{garg2021iq})}
\label{alg:iq}
\begin{algorithmic}[1]
\STATE Initialize Q-function $Q_\theta$, policy $\pi_\varphi$, and {\color{black}trajectory encoder} $E$
\WHILE{not converge}
   \STATE Reset the environment
   \STATE Sample a trajectory $\tau_i\sim \mathcal{D}$ and obtain its embedding $e_i=E(\tau_i)$
   \FOR{step $t$ in $\{1...T\}$}
       \STATE Train Q-function using obj. from Eq. 9 in~\citet{garg2021iq}:
       $$\theta_{t+1} \leftarrow \theta_t - \alpha_Q \nabla_\theta [-\mathcal{J}(\theta)]$$\\
       (Use $V^*$ for Q-learning and $V^{\pi_\varphi}$ for actor-critic)
       \STATE Improve policy $\pi_\varphi$ with SAC style actor update:
       \begin{align*}
           \varphi_{t+1} \leftarrow \varphi_t + \alpha_\pi \nabla_\varphi \mathbb{E}_{\vx\sim\beta, \va\sim\pi_\varphi(\cdot|\vx, \ve_i)}[&Q(\vx, \va| \ve_i) \\
           - &\log \pi_\varphi(\va|\vx, \ve_i)]
       \end{align*}
   \ENDFOR
\ENDWHILE
\end{algorithmic}
\end{algorithm}


\section{PCA Clustering Results}
We provide the clustering results using PCA in Figure \ref{fig:pca}.

\begin{figure*}[h]
    \centering
    \hspace*{-0.05\linewidth}  
    \includegraphics[width=1.1\linewidth]{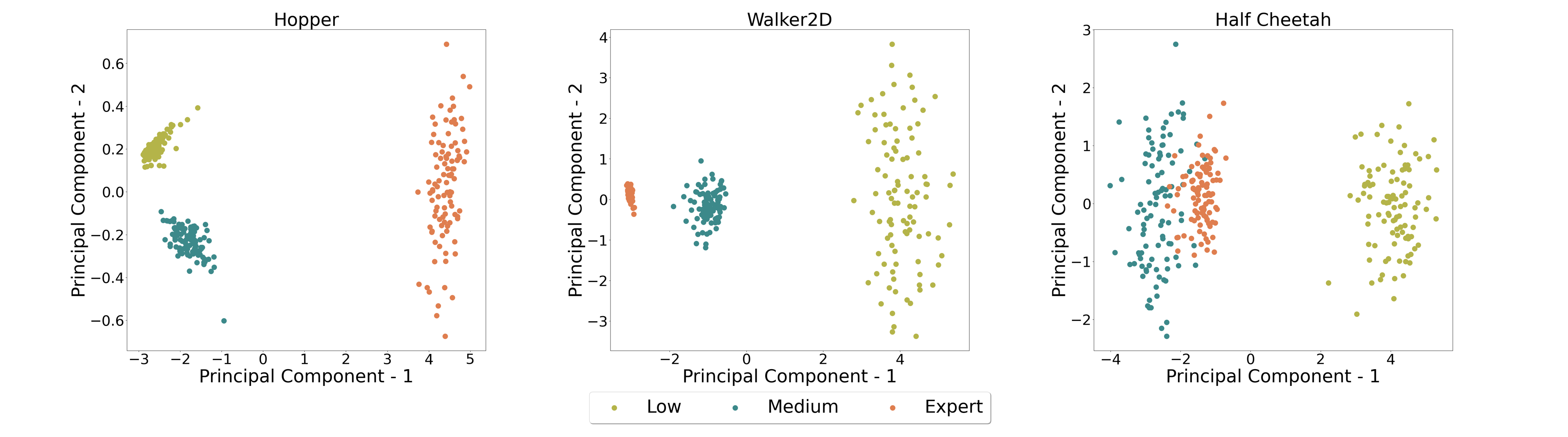}  
    \caption{Principal Component Analysis}
    \label{fig:pca}
\end{figure*}


\section{Training Graph}
We present the evaluation graphs for Walker2D and Half Cheetah respectively in Figure~\ref{fig:walker_train} and Figure~\ref{fig:cheetah_train}.

\begin{figure*}
    \centering
    \hspace*{-0.05\linewidth}  
    \includegraphics[width=1.1\linewidth]{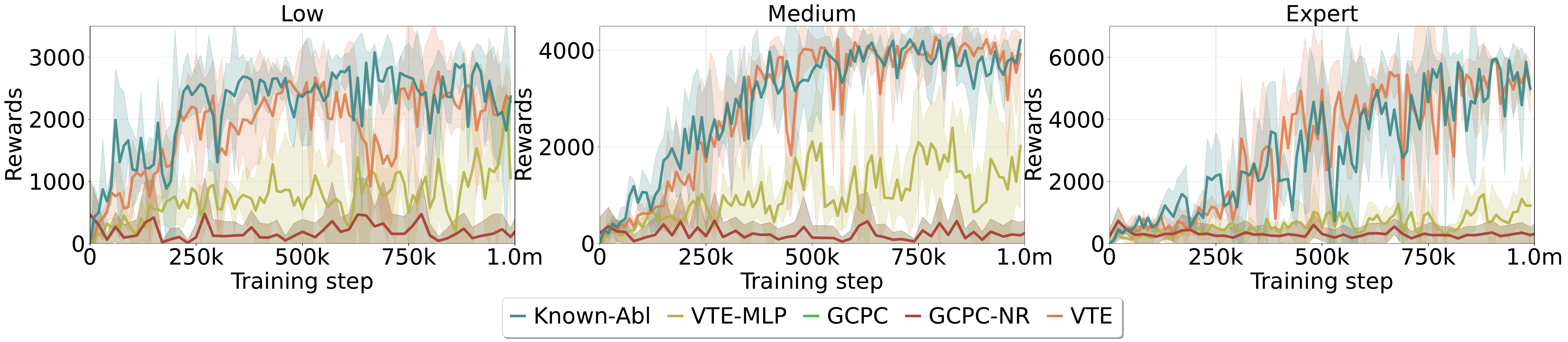}  
    \caption{Evaluation Graph on Walker2D}
    \label{fig:walker_train}
\end{figure*}

\begin{figure*}
    \centering
    \hspace*{-0.05\linewidth}  
    \includegraphics[width=1.1\linewidth]{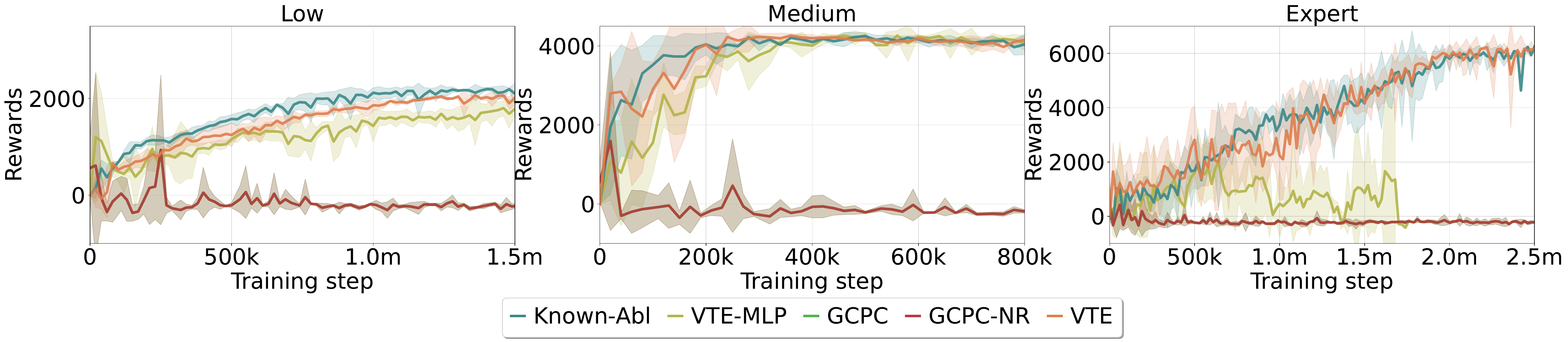}  
    \caption{Evaluation Graph on Half Cheetah}
    \label{fig:cheetah_train}
\end{figure*}

\section{Experiment Details}

\subsection{Optimization Details}
When learning skills through compression, the selection of hyperparameters has a limited effect on training outcomes. We assess performance by applying K-Means clustering to the mean pooling results of $\vz_{1:T}$. We use the Hierarchical State Space model proposed by Jiang et al.~\cite{jiang2022learning}, using a batch size of 16 and 200 training epochs. Early stopping is implemented once the clustering error falls below a sufficiently low threshold. This clustering error is determined by calculating the optimal assignment between the actual and predicted clusterings using the Hungarian algorithm. 
During Variational Trajectory Encoding, we adopt the hyperparameter settings outlined in Table~\ref{tab:vte_opt_details} for various environments. It is worth noting that bc\_alpha refers to the weight for the behavior cloning loss, which measures how closely the predicted actions match the expert actions, while kld\_alpha indicates the weight for the Kullback-Leibler divergence loss, which regularizes the learned latent distribution to resemble the prior distribution. We use a shallow transformer consisting of 4 attention blocks, each with a hidden size of 256, training the model for 500 epochs. We monitor the evolution of the trajectory embeddings and perform early stopping when the values show minimal variation.

\begin{table}[h]
    \centering
    \caption{Hyperparameter setting for VTE.}
    \begin{tabular}{ c c c }
        \toprule
        \textbf{Environment} & \textbf{bc\_alpha} & \textbf{kld\_alpha} \\
        \midrule
        Half Cheetah & 0.5 & 10 \\
        Hopper & 0.5 & 1.0 \\
        Walker2D & 0.5 & 1.0 \\
        \bottomrule
    \end{tabular}
    \label{tab:vte_opt_details}
\end{table}

For the downstream task of imitation learning, the total number of training timesteps for each environment is determined by the convergence of our model.  Generally, we utilize three different random seeds for each setting unless stated otherwise. Additional information can be found in Table \ref{tab:imitation_opt_details}. It is important to note that actor\_lr and critic\_lr indicates the learning rate of the actor and critic respectively, init\_temp refers to the initial temperature which controls the entropy of the policy, and batch size denotes the batch size during training updates.

\begin{table}[h]
    \centering
    \caption{Optimization details for imitation learning.}
    \begin{tabular}{ c c c c c }
        \toprule
        \textbf{Environment} & \textbf{actor\_lr} & \textbf{critic\_lr} & \textbf{init\_temp} & \textbf{batch size} \\
        \midrule
        Half Cheetah & 5e-5 & 1e-4 & 1e-12 & 32 \\
        Hopper & 3e-5 & 3e-5 & 1e-12 & 32 \\
        Walker2D & 2e-4 & 3e-5 & 5e-3 & 32 \\
        \bottomrule
    \end{tabular}
    \label{tab:imitation_opt_details}
\end{table}

\subsection{tSNE Clustering Analysis Parameters}
We employ the t-SNE implementation from scikit-learn with perplexity=3 and init=ramdom, while retaining default settings for all other parameters.
\subsection{Training Time}
Table~\ref{tab:training_time} shows our approach has similar training cost compared with baselines.
\begin{table}[h]
    \centering
    \caption{Average Training Time (hours) Across Various Environments.}
    \begin{tabular}{ c c c c }
        \toprule
        \textbf{Method} & \textbf{GCPC-NR} & \textbf{VTE-MLP} & \textbf{VTE} \\
        \midrule
        Training Time & 2 & 3 & 2.5 \\
        \bottomrule
    \end{tabular}
    \label{tab:training_time}
\end{table}

\section{Additional Results}

\subsection{Ablation Experiment}
For a given trajectory, the skill logits across all timesteps are averaged to produce the embedding, as described in~\cref{eq:pooling}. While this approach is straightforward, we hypothesize that the sequential structure of the skills is crucial for generating effective trajectory embeddings, which cannot be adequately captured by averaging. In our approach, we use a transformer within the encoder of the VAE, enabling us to incorporate the sequence information of the skills. 
The results in Table~\ref{tab:ablation_skill_pooling} of the paper show that mean pooling performs worse than our approach on the imitating task. Additionally, we observe that the trajectory embeddings of the naive approach lack diversity across the dimensions, suggesting that the embedding space is not fully utilized. This underutilization likely contributes to the inferior performance of mean pooling. Below, we provide some sample embeddings generated by both approaches.

\noindent\textbf{Mean pooling:} \\
Emb 1: 0.86, 0.85, 0.87, 0.86, 0.87, 0.88, 0.85, 0.87, 0.90, 0.85 \\
Emb 2: 0.83, 0.82, 0.83, 0.83, 0.83, 0.84, 0.82, 0.83, 0.86, 0.82 \\
Emb 3: 0.83, 0.82, 0.82, 0.83, 0.83, 0.84, 0.81, 0.83, 0.86, 0.82

\noindent\textbf{Ours: } \\
Emb 1: -0.84, 0.86, -0.97, -0.91, 1.05, 1.10, 0.88, 0.74, -0.79, 1.12\\
Emb 2: -0.83, 0.33, -0.57, -0.59, 0.68, 0.41, 0.94, 0.65, -0.33, 0.55\\
Emb 3: -0.86, 0.59, -0.78, -0.75, 0.87, 0.72, 0.92, 0.69, -0.57, 0.82

\subsection{Additional Environments}
 We conducted experiments on Pusher and HumanoidStandup. The results for the classification and regression tasks are in Table~\ref{tab:classification2} and Table~\ref{tab:regression2}. The performance on the imitation tasks in these two environments is suboptimal and we will include them in the future work.
\begin{table*}[h]
    \centering
    \caption{Classification Accuracy on Ability Levels(\%).}
    \begin{tabular}{ c c c c }
        \toprule
        \textbf{Environment} & \textbf{GCPC-NR} & \textbf{VTE-MLP} & \textbf{VTE} \\
        \midrule
        Pusher & 34.2 & 97.6 & 98.8 \\
        HumanoidStandup & 36.6 & 100.0 & 100.0 \\
        \bottomrule
    \end{tabular}
    \label{tab:classification2}
\end{table*}
\begin{table*}[h]
    \centering
    \caption{Relative Regression Error on Rewards(\%).}
    \begin{tabular}{ c c c c }
        \toprule
        \textbf{Environment} & \textbf{GCPC-NR} & \textbf{VTE-MLP} & \textbf{VTE} \\
        \midrule
        Pusher & 2.4 & 2.4 & 2.4 \\
        HumanoidStandup & 2.4 & 1.6 & 1.6 \\
        \bottomrule
    \end{tabular}
    \label{tab:regression2}
\end{table*}

\end{document}

%% file: math_commands.tex
\usepackage{amsmath,amsfonts,bm}

\def\eqref#1{equation~\ref{#1}}

\def\1{\bm{1}}

\def\va{{\bm{a}}}

\def\ve{{\bm{e}}}

\def\vm{{\bm{m}}}

\def\vs{{\bm{s}}}

\def\vx{{\bm{x}}}

\def\vz{{\bm{z}}}

\DeclareMathAlphabet{\mathsfit}{\encodingdefault}{\sfdefault}{m}{sl}
\SetMathAlphabet{\mathsfit}{bold}{\encodingdefault}{\sfdefault}{bx}{n}

\newcommand{\E}{\mathbb{E}}



%% file: components/main_figure.tex
\begin{figure*}[ht]
  \centering
  \includegraphics[width=0.95\linewidth]{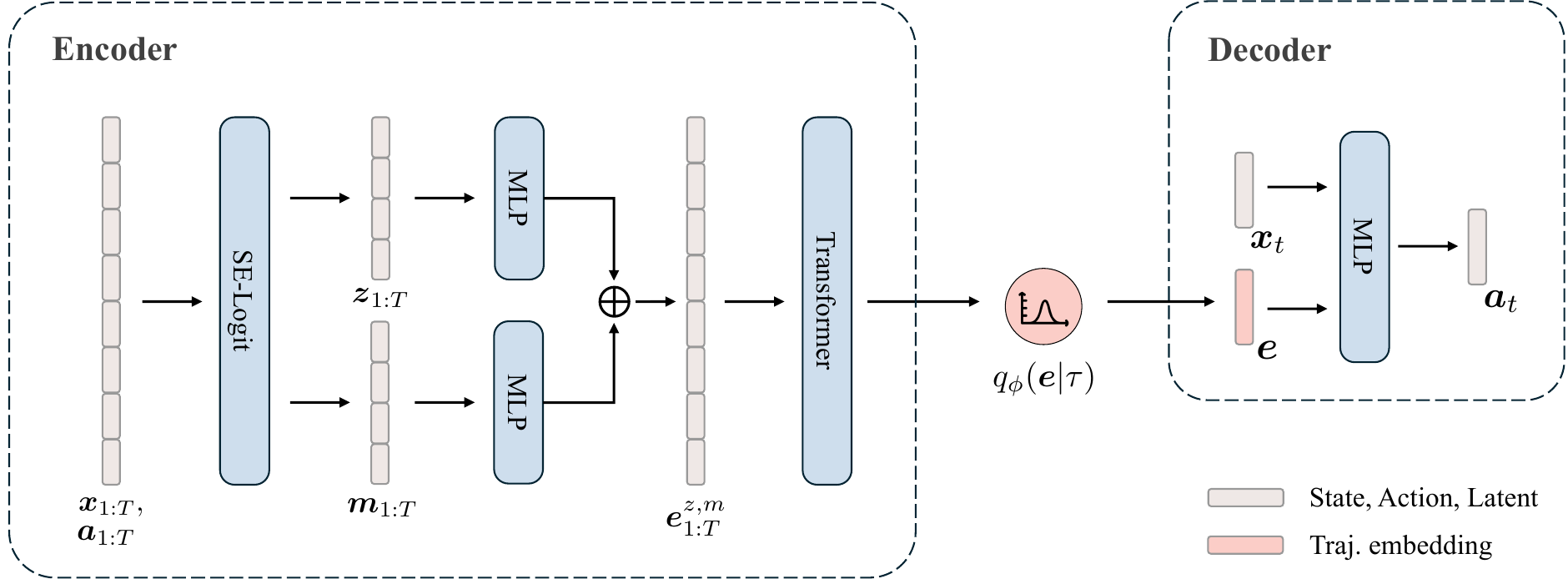}
  \caption{\textbf{Illustration of \ours{} Framework.} For the encoder, by exploiting the pretrained SE-Logit from~\cref{sec:love}, we extract the skill variable $\vz_{1:T}$ and the boundary variable $\vm_{1:T}$. These are then passed through separate MLPs, mapping $\vz_{1:T}$ and $\vm_{1:T}$ to $\ve^z_{1:T}$ and $\ve^m_{1:T}$, respectively, which are of equal size. At each time step, we concatenate these embeddings, resulting in $\ve^{z,m}_{1:T}$, where $\ve^{z,m}_{i}=\textrm{Concat}(\ve^z_i, \ve^m_i)$. $\ve^{z,m}_{1:T}$ is then fed into a transformer to compute the posterior $q_\phi(\ve|\tau)$. For the decoder, the action $\va_t$ is predicted from the state $\vx_t$, conditioned on the trajectory embedding.}
  \label{fig:main}
\end{figure*}